\documentclass[runningheads]{llncs}
\usepackage[numbers]{natbib}
\usepackage{graphicx}
\usepackage{amsmath}
\usepackage{xcolor}

\newif\ifcomments
\commentstrue

\ifcomments

\newcommand{\commentsm}[1]{\textcolor{blue}{(SM: #1)}}
\newcommand{\commentms}[1]{\textcolor{magenta}{(MS: #1)}}
\newcommand{\commenttk}[1]{\textcolor{cyan}{(TK: #1)}}
\newcommand{\old}[1]{\textcolor{red}{\sout{#1}}}

\else
 \newcommand{\commentsm}[1]{}
 \newcommand{\commentms}[1]{}
 \newcommand{\commenttk}[1]{}
 \newcommand{\old}[1]{}
 
\fi

\makeatletter
\newcommand{\nbcite}{\def\citeauthoryear##1##2{\def\@thisauthor{##1}%
\ifx \@lastauthor \@thisauthor \relax \else##1 \fi ##2}\@nbcite}
\def\citeS{\@ifnextchar[{\@jbciteS}{\@jbciteS[]}}
\def\@jbciteS[#1]#2{%
\ifthenelse{\equal{#1}{}}{%
\citeauthor{#2}'s (\citeyear{#2})}{%
\citeauthor{#2}'s #1 (\citeyear{#2})}}

\newcommand{\eat}[1]{}

\def\expl{\mathit{Expl}}
\def\bexpl{B_i\expl}

\begin{document}
\title{Towards the Role of Theory of Mind in Explanation}

\titlerunning{Towards the Role of Theory of Mind in Explanation}

\author{Maayan Shvo \and
Toryn Q. Klassen \and 
Sheila A. McIlraith}
%
\authorrunning{M. Shvo et al.}
%
\institute{Department of Computer Science, University of Toronto, Toronto, Canada\\Vector Institute, Toronto, Canada\\
\email{\{maayanshvo,toryn,sheila\}@cs.toronto.edu}}


\maketitle              
\begin{abstract}
Theory of Mind is commonly defined as the ability to attribute mental states (e.g., beliefs, goals) to oneself, and to others. A large body of previous work---from the social sciences to artificial intelligence---has observed that Theory of Mind capabilities are central to providing an explanation to another agent or when explaining that agent’s behaviour. In this paper, we build and expand upon previous work by providing an account of explanation in terms of the beliefs of agents and the mechanism by which agents revise their beliefs given possible explanations. We further identify a set of desiderata for explanations that utilize Theory of Mind. These desiderata inform our belief-based account of explanation.

\end{abstract}

\section{Introduction}\label{sec:intro}
%
%
Following \citeauthor{premack1978does} \cite{premack1978does}, an agent exercises 
\textit{Theory of Mind} if it imputes mental states to itself and others. 
Here we explore the role of Theory of Mind in explanation.  
Consider the following narrative by way of illustration.

\begin{quote}
{\it Mary, Bob and Tom are housemates sharing a house. While Tom was away on a business trip, Mary and Bob noticed a hole in the roof of their house and called a handyman to fix it. Before the handyman could come, however, it rained during the night and the floor got wet. Bob, who sleeps in a windowless room, did not notice the rain. Tom, who just got back from his trip that day, noticed the rain but did not know about the hole in the roof. Mary saw Tom return to the house at night and so knew that Tom knew that it had rained. In the morning, when trying to explain the wet floor to Bob, Mary tells him that it had rained during the night and when explaining to Tom she tells him that she and Bob had discovered a hole in the roof (adding that the handyman will arrive the next day). 
}
\end{quote} 


Clearly, Mary tailored her explanations to each of her housemates, 
believing the information she  was providing to them was sufficient to explain the wet floor in their respective mental states. Her ability to do this stems from her Theory of Mind - her ability to attribute mental states (e.g., beliefs) to herself and to others. In humans, the use of Theory of Mind in explanation has been demonstrated empirically by \citeauthor{slugoski1993attribution} \cite{slugoski1993attribution} via a set of experiments where human participants gave different explanations to different explainees (i.e., the recipient of an explanation), based on the beliefs of the explainers about the beliefs of the explainees\footnote{We henceforth use \emph{explainer} and \emph{explainee} in reference to the provider and recipient of the explanation, and \emph{explanandum} in reference to the thing to be explained.}.
%
%
Of course Mary's explanations are only as good as her ability to model the mental states of her housemates and how they will alter their mental states in light of her explanation. Mary's beliefs about Bob and Tom's beliefs, or her belief about how each of them revises their beliefs,
may well be wrong, in which case her explanations to them may fail to explain why the floor is wet. 

Explanation has been studied in a diversity of disciplines.  \citeauthor{DBLP:journals/ai/Miller19} \cite{DBLP:journals/ai/Miller19} provides an extensive survey of explanation in artificial intelligence that includes a selection of historical works in philosophy (e.g., \citeauthor{hempel1948studies} \cite{hempel1948studies}; \citeauthor{peirce1878deduction} \cite{peirce1878deduction}; \citeauthor{harman1965inference} \cite{harman1965inference}), arguing for the important role of philosophy and the social sciences in future work on explanation. Within AI, early work on explanation included a variety of logic-based and probabilistic approaches to abductive inference or so-called {\it inference to the best explanation} including the early works of  \citeauthor{pople1973mechanization} \cite{pople1973mechanization}, \citeauthor{charniak-book85}
\cite{charniak-book85},
\citeauthor{poole1989explanation}
\cite{poole1989explanation}, and
\citeauthor{levesque1989knowledge}
\cite{levesque1989knowledge}. 
In the mid 1980s, explanation was popularized in the context of expert systems where explanations were often generated by backward chaining over a set of symbolic inference steps 
(e.g., \cite{hayes1983building,shortliffe1985rule}). 
Following that time, explanation was a common element in a diversity of applications of symbolic AI reasoning (e.g., \cite{mcg-pin-jws04,bor-cal-rod-dl08,sohrabi2011preferred}). The recent resurgence of interest in explanation is largely in the guise of so-called \emph{Explainable AI} (XAI), which is motivated by the need to provide human-interpretable explanations for decision making in black-box classification and decision-making systems based on machine and deep learning (e.g., \citeauthor{samek2017explainable} \cite{samek2017explainable}; \citeauthor{GunningSCMSY19} \cite{GunningSCMSY19}).

Numerous researchers have acknowledged the importance of Theory of Mind in explanation. In the 80s and 90s, formal accounts of explanation such as those proposed by \citeauthor{gardenfors1988knowledge} \cite{gardenfors1988knowledge} and \citeauthor{chajewska2013defining} \cite{chajewska2013defining} observed that an explanation for one agent may not serve as an explanation for another, and the explainer must therefore tailor an explanation to an explainee given the latter's beliefs.
Within the space of user modelling and dialogue, and also set in the 80s and 90s, \citeauthor{weiner1980blah}'s \cite{weiner1980blah} BLAH system and \citeauthor{cawsey1991generating}'s \cite{cawsey1991generating} EDGE system both tailor explanations to the presumed user model.
More recently, researchers have leveraged belief-desire-intention (BDI) architectures as a natural framework for explanations reflecting Theory of Mind. Such software architectures can enable an explainer to explicitly represent its own beliefs, desires, and intentions, as well as those of an explainee, and to relate explanations to its own beliefs and goals or those of the explainee (e.g., \citeauthor{harbers2012modeling} \cite{harbers2012modeling}; \citeauthor{kaptein2017personalised} \cite{kaptein2017personalised}). Most recently, \citeauthor{westberg2019historical} \cite{westberg2019historical} has posited that incorporating various points of view on Theory of Mind from the cognitive sciences will facilitate the creation of agents better suited to communicate and explain themselves to the humans with whom they are interacting. Additionally, \citeauthor{DBLP:journals/ai/Miller19} \cite{DBLP:journals/ai/Miller19} has surveyed this body of work and has also emphasized the importance of the explainer's ability to tailor an explanation to the explainee, using its understanding of the latter's mind.
Finally, within the subfield of XAI known as XAI Planning (XAIP) \citeauthor{chakraborti2017plan} \cite{chakraborti2017plan} have implemented XAIP in human-agent teaming settings, such as search \& rescue, where a robot equipped with Theory of Mind capabilities could explain its actions to its human teammate by taking into account the latter's mental state.

In this paper we build on the shoulders of previous scholarly work to explore the role of Theory of Mind in explanation with a view to addressing the diverse needs of explanation in AI, and XAI in particular. To this end, in Section~\ref{sec:desiderata} we identify a set of desiderata for explanations that utilize Theory of Mind. These desiderata inform a set of design choices for a belief-based account of explanation which we present in Section~\ref{sec:expl}. Of course not all explanations are created equal, and in Section~\ref{sec:best_expl} we discuss the criteria by which the quality of an explanation can be evaluated. In Section~\ref{sec:discrepancy} we demonstrate how, in the absence of an explicit prompt to be explained, 
our account allows the explainer to simulate the explainee's mental state and identify discrepancies that warrant explanation. Explanations are limited by the coverage and accuracy of the explainer's beliefs as well as its reasoning capacity. In Section~\ref{sec:adeq}, we show how our account allows for the modelling of the ignorance and misconceptions of an explainer pertaining to the mental state of an explainee and how these may affect the quality of explanation. We conclude 
with a discussion of related work and possible computational realizations of our general account.

\section{Desiderata for Theory of Mind in Explanation}\label{sec:desiderata}



We begin our investigation by reflecting on the key components that support an agent in imputing mental states to itself and others, reasoning about how the provision of new information is assimilated into an agent's existing set of beliefs, and the circumstances underwhich such information constitutes an explanation for the explainee. 
To this end, we identify a set of desiderata that inform our account of explanation in the sections to follow.

%
\begin{description}
\item [multi-agent:] the account must be conceived in a multi-agent setting to support representation of the beliefs of one or more explainer and explainee.
\item[agent-type agnostic:] the account must support a myriad of different agent types whose beliefs may be internally represented, inspectable, and revisable in diverse ways. For example,   
the agent's beliefs may be stored in a human brain or in, for instance, the parameters of a neural network or formulae in a knowledge base.
\item[belief based:] the account must model the possibly false or simply incomplete beliefs of explainers and explainees. 
%
%
\item[reason about the beliefs of others:] the account must allow an explainer to reason about the explainee's beliefs when providing the latter with an explanation since, due to their possibly differing beliefs, an explanation for the explainer may not be an explanation for the explainee.
%
%
\item[support belief revision:] the account must enable the explainer to consider how an explanation is assimilated by the explainee, and in particular how the latter revises their beliefs given potential explanations which may be inconsistent with their current beliefs.
\item[explanations can refer to beliefs:] the account must allow for explanations that themselves refer to beliefs. To illustrate why this is useful, consider that the explainer might explain their having not told the explainee the location of a party by saying that the explainer believed that the explainee knew the location. 
\end{description}

While previous work has addressed some of these desiderata, in this paper we propose a belief-based account of explanation in terms of epistemic states of agents that satisfies all of the aforementioned desiderata by employing a number of crucial building blocks relating to these desiderata.

\section{A Belief-Based Account of Explanation}\label{sec:expl}



We appeal to logics of belief to provide a belief-based account of explanation in the context of Theory of Mind.
%
%

Many logical accounts of explanation assume the existence of a knowledge base---a logical axiomatization of the domain in terms of a set of formulae~(e.g., \cite{lev-bra-book04}). With such a knowledge base in hand, a popular logic-based characterization of explanation is in terms of abduction
as follows.


\begin{definition}[Abductive Explanation (after \cite{poole-ijis90})]\label{def:abduction}
Given a logical theory, $T$, and an explanandum $O$, $E$ explains $O$ given a theory $T$ if
$T\cup E\models O$ and $T\cup E$ is consistent. 
\end{definition}



Here we make no such commitment to the representation of beliefs in terms of a set of logical formulae. Rather, in order to capture the diversity of human and machine explainers and explainees, our account finds its origins in works that attributed agents with mental states in the form of epistemic states (with seminal work by \citeauthor{gardenfors1988knowledge} \cite{gardenfors1988knowledge} and later notable work by \citeauthor{levesque1989knowledge} \cite{levesque1989knowledge}; \citeauthor{boutilier1995abduction} \cite{boutilier1995abduction}; \citeauthor{chajewska2013defining} \cite{chajewska2013defining}; and \citeauthor{halpern2005causes} \cite{halpern2005causes}).
%




\subsection{Mental States as Epistemic States}
We employ the notion of an epistemic state, $e$, or in the case of multiple agents, a collection of epistemic states, $\vec e$, to capture the beliefs of agents. These are used to provide the semantics for the language below.

We will suppose that we have a finite set of agents, $A=\{1,2,\dots,n\}$, and a set of propositional symbols $P$.
We define a language
\begin{align}
    \varphi ::= p\ | \neg\varphi\ |\ (\varphi\wedge\varphi)\ |\ B_i\varphi\ |\ [\varphi]_i\varphi\label{bnf}
\end{align}
where $p\in P$ and $i\in A$. 
We introduce $\bot$ as an abbreviation for $(p\wedge\neg p)$ for an arbitrary $p\in P$. 

The intended meaning of $B_i\varphi$ is that agent $i$ believes $\varphi$, and the intended meaning of $[\alpha]_i\varphi$ is that after agent $i$ revises their beliefs by $\alpha$, $\varphi$ is true.


We assume that our epistemic states are such that we can say that a formula $\varphi$ is true at $e$ when $\varphi$ is believed. 
To be clear, although we use formulas to describe what is believed, an epistemic state is not in general \emph{defined} as a set of formulas, nor required to be represented internally as one. 
For a conventional example, $e$ might be a set of possible worlds with accessibility relations and so on. However, we also allow for epistemic states to take very different forms. For example, one might want to model limited reasoning capabilities in some manner to avoid the so-called problem of logical omniscience \cite{Stalnaker1991}, in which agents unrealistically believe all the deductive consequences of their beliefs. We might also wish for our epistemic states to be realized in terms of a computer program, such as a neural network, or via a human brain. 

Furthermore, we assume we have a \textit{revision operator} $*$ so that $e*\alpha$ is another epistemic state, the result of revising by $\alpha$. We will use * in defining the semantics for the $[\alpha]_i$ operator. Much as we have not committed to a particular structure for epistemic states, we will not commit to a particular revision operator. A large body of work has studied belief change in agents where belief revision typically concerns belief change in a static environment, possibly in the presence of incorrect and partial beliefs. Amongst the most popular guidelines for belief revision are the AGM postulates \cite{alchourron1985logic}, 
and the DP postulates  \cite{darwiche1997logic} (for iterated revision). We will not require that our $*$ satisfies these properties except where noted. Similarly to the situation with our epistemic states, we might want our revision operator to be realized in terms of a computer program  or human reasoning. 


While epistemic states assign a truth value to any formula in our language -- the language given by the grammar in (\ref{bnf}) -- that value indicates whether the formula is believed by the agent in question, not whether it's actually true. From an objective point of view, the formulas whose truth values we can determine are from the subset of the language consisting of formulas which are concerned only with beliefs. We define this subset of formulas below:
\begin{definition}[Agent Formula]
An agent formula is one in which no atomic symbol appears outside the scope of a belief operator, i.e., a formula $\phi$ of the form
\begin{align}
    \phi::= B_i\varphi\ |\ \neg\phi\ |\ (\phi\wedge\phi)\ |\ [\varphi]_i\phi
\end{align}
where $\varphi$ is any (possibly non-agent) formula.
\end{definition}
We assign truth values to agent formulas with a collection of epistemic states $\vec e =e_1,\dots,e_n$ (corresponding to the different agents) according to the satisfaction relation $\models$ below.
\begin{itemize}
    \item $\vec e\models B_i\varphi$ iff $\varphi$ is true at $e_i$
    \item $\vec e \models \neg\phi$ iff $\vec e\not\models\phi$
    \item $\vec e\models (\phi\wedge\psi)$ iff $\vec e\models\phi$ and $\vec e\models\psi$
    \item $\vec e\models [\alpha]_i\phi$ iff $\langle e_1,\dots, (e_i*\alpha), \dots, e_n\rangle\models\phi$
\end{itemize}
Note that the semantics of the $[\alpha]_i$ operator is defined using the revision operator.

Give this abstract framework for talking about beliefs, we can define explanations.
The lack of commitment to the form of the epistemic state and revision operator is important because it affords us the ability to model a diversity of agents. In so doing, for the definitions of explanation that follow, the explainer will have beliefs about the other agents' beliefs and about their revision operators, and the effectiveness of the explainer's explanations for any particular agent will rely on the fidelity of those beliefs.

\subsection{Characterizing Explanations}

\begin{definition}[Explanation]\label{def:expl}
Given epistemic states $\vec e$, we say that $\alpha$ explains $\beta$ for agent $i$ if $\vec e\models [\alpha]_i (B_i\beta\wedge\neg B_i\bot)$.
\end{definition}

\noindent {\bf Notation:}
For notational 
convenience, we define $\expl(i,\alpha,\beta)$ as an abbreviation for $[\alpha]_i (B_i\beta\wedge\neg B_i\bot)$.\\

That is, $\alpha$ explains $\beta$ if revising by $\alpha$ makes agent $i$ believe $\beta$ while still having consistent beliefs.\footnote{
If agent $i$ is not logically omniscient, requiring $i$ to not believe $\bot$ may not prevent $i$'s beliefs from being inconsistent in some subtler way. For example, $i$ might both believe $p$ and believe $\neg p$, even though it does not believe $(p\wedge\neg p)$.} Note that (with respect to revising by non-modal formulas) if revision of agent $i$'s epistemic state satisfies the AGM postulates, then the result of revision will be inconsistent only if either the agent initially had inconsistent beliefs, or if $\alpha$ itself is inconsistent. 

Intuitively, our definition of explanation allows for more explanations than the traditional account in Definition~\ref{def:abduction}. For one thing, we allow explanations to refer to modal operators. Even without that, though, an important difference is that our definition is in terms of belief revision and so allows for an explanation that isn't consistent with the agent's initial beliefs. Our account builds upon prior accounts of explanation defined relative to belief revision such as \citeauthor{boutilier1995abduction} \cite{boutilier1995abduction} and \citeauthor{nepomuceno2017abductive} \cite{nepomuceno2017abductive}.


To make the comparison more explicit, consider defining an epistemic state $e_i$ as a  propositional theory $T$, as in the following theorem.
\begin{theorem}\label{thm:subsume}Suppose that $e_i$ is defined as being a propositional theory $T$, and that the formulas $e_i$ makes true are defined to be the logical consequences of $T$ (note that these are restricted to the non-modal subset of our language). 
Suppose furthermore that the revision operator $*$ on $e_i$ satisfies the AGM postulates (w.r.t. non-modal formulas).
Then for non-modal formulas $\alpha$ and $\beta$, $\vec e\models \expl(i,\alpha,\beta)$ if $T\cup \{\alpha\}$ is consistent and $T\cup\{\alpha\}\models \beta$.
\end{theorem}
\begin{proof}
Because $T\cup\{\alpha\}$ is consistent, by the AGM ``vacuity'' postulate, $T*\alpha$ is equal to the expansion of $T$ by $\alpha$, that is, the closure of $T\cup\{\alpha\}$. Therefore, $T*\alpha\models \beta$.
\end{proof}
However, we may also get further explanations. In the circumstances described by Theorem~\ref{thm:subsume}, if $T\cup\{\beta\}$ is inconsistent, then Definition~\ref{def:abduction} would say there are no explanations of $\beta$ given the theory $T$, while there may be formulas that agent with epistemic state $T$ can revise by that would make them believe $\beta$.




It is also possible to talk in the language about agents' beliefs about $\expl(i,\alpha,\\\beta)$, i.e. about whether $\alpha$ explains $\beta$ for agent $i$. 


\begin{definition}[Subjective Explanation]
Given epistemic states $\vec e$, we say that $\alpha$ explains $\beta$ for agent $j$ from agent $i$'s perspective, if $\vec e \models \bexpl(j,\alpha,\beta)$.

\end{definition}







\begin{example}
We formalize our example from Section~\ref{sec:intro}. We assume that Mary, Bob and Tom all believe (and believe that the other agents believe) $rain \land holeInRoof \rightarrow wetFloor$.


\noindent\makebox[\textwidth]{\rule{4.8in}{0.4pt}}
\vspace{-1.5em}
\begin{itemize}\itemsep1em
    \item[] $A = \{ \textit{Mary}, \textit{Bob}, \textit{Tom} \}$ 
    
    \item[] $\vec e \models B_{Mary} wetFloor \land B_{Mary} holeInRoof \land B_{Mary} rain$
    
    \item[] $\vec e \models B_{Mary} B_{Bob} \lnot wetFloor \land B_{Mary} B_{Bob} \lnot rain \land B_{Mary} B_{Bob} holeInRoof$
    
    \item[] $\vec e \models B_{Mary} B_{Tom}  \lnot wetFloor \land B_{Mary} B_{Tom} rain \land B_{Mary} B_{Tom} \lnot holeInRoof$
    
    
    
    \item[] $\vec e \models B_{Mary}\expl(Bob,rain,wetFloor)$
    
    \item[] $\vec e \models B_{Mary}\expl(Tom,holeInRoof,wetFloor)$
    
\end{itemize}
\vspace{-1em}
\noindent\makebox[\textwidth]{\rule{4.8in}{0.4pt}}

We also assume that the agents are able to draw at least simple inferences (and each knows that the others will) and their belief revision operators behave in a sensible way (and each knows that the others' operators do so).

\end{example}


We define a relation $\approx$ that can be understood intuitively as equating two epistemic states, $e_i$ and $e_j$. For $e_i \approx e_j$ to hold, 
the internal structures of the states $e_i$ and $e_j$ need not be the same, but they must support the same beliefs as each other, and must continue to do so after any sequence of revisions.
%
%
%
Formally, we say that $e_i \approx e_j$ if
\begin{itemize}
    \item $\vec e \models B_i \varphi$ \textit{iff} $\vec e \models B_j \varphi$ 
    \item and for any sequence of formulas $\alpha_1,\dots,\alpha_k$, we have that $\vec e\models [\alpha_1]_i\cdots [\alpha_k]_i B_i\varphi$ iff $\vec e\models [\alpha_1]_j\cdots [\alpha_k]_j B_j\varphi$ 
\end{itemize}



\begin{theorem}
Given epistemic states $\vec e$ and explanandum $\beta$, if $e_i \approx e_j$ it then follows that for all $\alpha$, $\vec e \models \expl(i,\alpha,\beta)$ iff $\vec e \models \expl(j,\alpha,\beta)$.
\end{theorem}
\begin{proof}
Note that $\vec e\models \expl(i,\alpha,\beta)$ iff $\vec e\models [\alpha]_i B_i\beta$ and $\vec e\models [\alpha]_i \neg B_i\bot$, and similarly for agent $j$. The result follows from the definition of $\approx$.
\end{proof}




That is, when $e_i \approx e_j$, an objective explanation for the former is also an objective explanation for the latter. Therefore, agent $i$, acting as the explainer, need not employ its Theory of Mind and reason about agent $j$'s beliefs in order to generate explanations for the latter. However, the fact that $e_i \approx e_j$ does not mean that $e_i$ holds accurate beliefs pertaining to how $e_j$ revises its beliefs. Thus, while any $\alpha$ that explains $\beta$ may be an objective explanation for both agents $i$ and $j$, agent $i$ need not necessarily \textit{believe} that $\alpha$ is an explanation for $j$.
Nonetheless, $e_i\approx e_j$ is quite strong, 
as illustrated by the following theorem.
\begin{theorem}
Suppose $e_j$ supports positive and negative introspection -- i.e., $\vec e\models (B_j \varphi \equiv B_jB_j\varphi)\wedge (\neg B_j\varphi \equiv B_j\neg B_j\varphi)$. Then if $e_i\approx e_j$, agent $i$ will have correct beliefs about $j$'s beliefs, i.e., $\vec e \models ( B_j\varphi\equiv B_iB_j\varphi )\wedge (\neg B_j\varphi\equiv B_i\neg B_j\varphi)$.
\end{theorem}
\begin{proof}
If agent $j$ believes $\varphi$, then we'll have that $\vec e\models B_j B_j \varphi$ (by positive introspection) and then $\vec e \models B_i B_j \varphi$ (because $i\approx j$). Similarly, if agent $j$ disbelieves $\varphi$, then  $\vec e\models B_j \neg B_j \varphi$ (by negative introspection) and so $\vec e \models B_i \neg B_j \varphi$.
\end{proof}

In some cases, an explanation need not cause the explanandum to be entailed by the epistemic state, but rather cause it to be \textit{possible} in the epistemic state. This type of explanation is similar to \citeauthor{boutilier1995abduction}'s \textit{might explanation}.

\begin{definition}[Inconsistency-resolving Explanation]\label{def:might_expl}
Given epistemic states $\vec e$, we say that $\alpha$ explains the possibility of $\beta$ for agent $i$ if $\vec e\models[\alpha]_i\neg B_i\neg \beta$.
\end{definition}

 This is a weaker form of explanation but important in various settings such as when an agent is attempting to find an explanation that will allow the behavior of another agent or in consistency-based diagnosis, where the agent is attempting to identify the abnormal componenets in a system that allow for the observed behavior of the system.

\begin{theorem}
Given epistemic states $\vec e$ and explanandum $\beta$, then for all $\alpha$, if $\vec e \models \expl(i,\alpha,\beta)$ it then follows that $\alpha$ is an inconsistency-resolving explanation for $\beta$ for agent $i$, assuming that $\vec e \models [\alpha]_i\big((B_i \beta \land B_i \lnot\beta) \rightarrow B_i\bot\big)$, i.e., that the agent can perform enough reasoning to notice the inconsistency in believing both $\beta$ and $\neg\beta$.
\end{theorem}

This follows straightforwardly from Definitions~\ref{def:expl} and \ref{def:might_expl}.\\

\noindent\textbf{Explanations Involving Agent Beliefs}\\
%
%
Importantly, an explainer can utilize its Theory of Mind to generate explanations pertaining to the mental states of other agents, such as their beliefs or goals.

\begin{example}
Let us reconsider our example where this time, after Mary explains $wetFloor$ to Bob, he asks her why Tom doesn't know $wetFloor$. That is, the explanandum $\beta$ is $\lnot B_{Tom}wetFloor$. A possible explanation is then\\ $B_{Tom} \lnot holeInRoof$, assuming Bob believes $B_{Tom} rain$.
\end{example}



\subsection{Explanations Involving Multiple Agents}


An interesting setting that is 
straightforwardly captured by our framework is one in which an explainer (or explainers) is attempting to explain multiple (possibly disparate) explanandums to multiple explainees. 




\begin{definition}\label{def:multi_agent_expl}
Given epistemic states $\vec e$ and explanandums $\beta_j$, $\beta_k$, \ldots $\beta_l$, we say that $\alpha$ explains $\beta_j$, $\beta_k$, \ldots $\beta_l$ from agent $i$'s perspective for agents $j$, $k$, \ldots $l$, respectively, if $\vec e \models \bexpl( j, \alpha,\beta_j) \land \bexpl( k, \alpha,\beta_k) \land \ldots \land \bexpl( l, \alpha,\beta_l)$.
\end{definition}

Consider a collaborative card game (e.g., Hanabi \cite{bard2020hanabi}) where a certain player is attempting to make different players (each with a unique epistemic state) understand different things with a single piece of information about another player's cards, publicly announced to all players. The explaining player should therefore find an $\alpha$ that explains different explanandums for the different players, given the explaining player's beliefs about the other players' beliefs.

\begin{example}
In a simpler setting such as our running example, if Mary is trying to explain $wetFloor$ to Bob and Tom at the same time, the explanation $\alpha$ could be $rain \land holeInRoof$, where the explanandum for both Bob and Tom is $wetFloor$.
\end{example}


\subsubsection*{Privacy}

Our framework can also capture a notion of privacy. For example, the explainer (agent $i$) may want to generate an explanation $\alpha$ that explains the explanandum $\beta$ to some agents (agent $j$) but not to others (agent $k$):
\begin{align*}
\vec e \models \bexpl( j, \alpha,\beta) \land B_i\lnot(\expl(k, \alpha,\beta))
\end{align*}


\begin{example}
If Mary, for some reason, wants only Bob to entail $wetFloor$, the explanation $\alpha$ could be $rain$ in which case Bob will entail $wetFloor$ but Tom will not. One can imagine parent \#1 wanting to explain something to parent \#2 such that their child does not understand.
\end{example}


\subsubsection*{Multiple Explainers and `Nested' Explanations}

In some cases, there may be multiple explainers trying to explain an explanandum $\beta$ to an explainee. For example, agents $i$ and $j$ may want to find an $\alpha$ that explains $\beta$ for agent $k$:
\begin{align*}
\vec e \models \bexpl( k, \alpha,\beta) \land B_j\expl( k, \alpha,\beta)
\end{align*}

Definition~\ref{def:multi_agent_expl} can be straightforwardly extended to capture this setting. Finally, agent $i$ may want to find an $\alpha$ that he believes that agent $j$ believes is an explanation for agent $k$:
\begin{align*}
\vec e \models B_i B_j\expl( k, \alpha,\beta)
\end{align*}


\section{``Best'' Explanations for Whom?}\label{sec:best_expl}

An explanadum can typically be explained by a variety of different explanations, but it is often the case that an agent \textit{prefers} one explanation to another relative to some set of criteria. Indeed, there is a large body of previous work (e.g., \cite{levesque1989knowledge,lipton1990contrastive,boutilier1995abduction}) that outlines criteria for defining preference orderings over explanations. 
In the context of a multiple agents, we have seen that what constitutes an explanation for one agent, may not constitute an explanation for another. This observation extends to the notion of preferred explanations---what's good in the eyes of the explainer may not be good for the explainee, or for all explainees. We explore the issue of preferred explanations briefly here in the context of Theory of Mind.

%


For each agent in the set of agents $A$, we define a binary preference relation $\prec$ over explanations such that $\prec_i$ is the preference relation for agent $i$.

\begin{definition}[Preferred Explanation]
Given epistemic states $\vec e$ and explanandum $\beta$, if $\alpha$ and $\alpha'$ both explain $\beta$ for agent $i$ and $\alpha \preceq_{i} \alpha'$, we say that $\alpha$ is at least as preferred as $\alpha'$ for agent $i$. $\alpha \prec_{i} \alpha'$ denotes that $\alpha$ is strictly preferred to $\alpha'$ for agent $i$.
\end{definition}

Similarly, we use $\alpha \preceq_{i,j} \alpha'$ to denote that agent $i$ believes that $\alpha$ is at least as preferred as $\alpha'$ for agent $j$. 



\begin{definition}[Optimal Explanation]\label{def:optimal_expl}
Given epistemic states $\vec e$ and explanandum $\beta$, $\alpha$ is an optimal explanation for $\beta$ wrt $\prec_i$ iff $\alpha$ explains $\beta$ for agent $i$ and there does not exist an explanation $\alpha'$ for $\beta$ for agent $i$ such that $\alpha'$ $\prec_i$ $\alpha$.
\end{definition}

 \citeauthor{hilton1990conversational} \cite{hilton1990conversational} posits that an explanation given by one agent to another is a form of conversation and should therefore adhere to \citeauthor{grice1975logic}'s \cite{grice1975logic} maxims which he proposed as part of a model for effective cooperative conversation. In what follows, we discuss a number of criteria for preferred explanations and relate them to Grice's maxims.

\subsubsection*{Truthfulness:}

\citeauthor{grice1975logic}'s first maxim is the \textbf{quality} maxim, according to which one must not provide information (e.g., to the explainee) that she believes to be false. 


\begin{definition}[Subjectively Truthful Explanation]
Given epistemic states $\vec e$ and an explanandum $\beta$, $\alpha$ is a subjectively truthful explanation for agent $j$ from the perspective of agent $i$ iff $\vec e \models \bexpl(j,\alpha,\beta)\wedge B_i\alpha$.

\end{definition}

\begin{example}
In our example, Mary may tell Bob that Tom poured water all over the floor, thereby explaining \textit{wetFloor}. However, since Mary does not believe that Tom did such a thing, it would not be a subjectively truthful explanation explanation from Mary's perspective.
\end{example}

%

\subsubsection*{Minimality:}
According to Grice's \textbf{quantity} and \textbf{relation} maxims, one must provide information that is relevant, sufficiently informative, and no more informative than needed. In a Theory of Mind context, the sufficiency of information is defined relative to the explainer's beliefs about the explainee's epistemic state and the explainer should therefore find the \textit{minimal} explanation relative to the explainee's epistemic state.
%
%
%
A large body of work concerned with explanation has discussed a minimality property which an explanation should satisfy. 
%
%
For example, \citeauthor{levesque1989knowledge} \cite{levesque1989knowledge} defines a syntactic simplicity relation between explanations wherein an explanation is \textit{simpler} than another if it contains fewer propositional letters. Minimal explanations in the semantic sense may be defined relative to a set of possible explanations as those that are implied by all other explanations.

\subsubsection*{Plausibility:}
\citeauthor{grice1975logic}'s \textbf{quality} maxim also dictates that one should not provide information that is not supported by evidence. When applying this maxim to the beliefs of the explainee, an explainer may wish to consider how likely an explanation is from the point of view of the former.
For instance, in our example it is more likely that Bob will accept \textit{rain} as an explanation over the highly unlikely explanation according to which Alan Turing came to visit in the middle of the night and accidentally poured water all over the floor. Therefore, the likelihood of an explanation is an important preference criterion when explaining to ourselves and to others. In the quantitative case, \citeauthor{pearl2014probabilistic} \cite{pearl2014probabilistic} defines a \textit{most probable explanation} while in a qualitative setting the \textit{plausibility} of explanations may be defined where the most plausible explanations are those that require the `least' change in the explainee's epistemic state (e.g., \cite{quine1978web,boutilier1995abduction}), which could be defined in various ways, including the degree of held beliefs (e.g., \cite{van1992graded}).

\section{Explainer-Explainee Discrepancies}\label{sec:discrepancy}


%
To this point our account of explanation has assumed the existence of an explanandum, $\beta$, that is in need of explanation for a particular agent. However, in the absence of such a prompt, the explainer may use her Theory of Mind to put herself in the explainee's shoes, so to speak, and to identify \textit{discrepancies} between the beliefs of the explainee and those of the explainer, or perhaps in the case of multiple agents, to identify discrepancies between the beliefs of two agents that the 
explainer can resolve via an explanation. Discrepancies can also arise from inconsistencies between an agent's beliefs and observations in the world. Such discrepancies are common prompts for explanation in the case of diagnosis (e.g., \cite{Reiter87,boutilier1995abduction}).

%

%





\begin{definition}[Discrepancy]
Given epistemic states $\vec e$, $\beta$ is a discrepancy between $e_i$ and $e_j$ iff $\vec e \models B_i\beta$ $\land$ $B_j \lnot\beta$.
\end{definition}

That is, agent $i$ believes $\beta$ while agent $j$ believes $\lnot\beta$. The beliefs of agents pertaining to discrepancies can also be represented in our framework.

\begin{definition}[Subjective Discrepancy]\label{def:subj_disc}
Given epistemic states $\vec e$, $\beta$ is a discrepancy between $e_i$ and $e_j$ from the perspective of agent $i$ iff $\vec e \models B_i(B_i\beta$ $\land$ $B_j \lnot\beta)$.
\end{definition}





\begin{example}
In our example, while Mary believes \textit{wetFloor}, she believes that Bob believes that the floor is not wet (i.e., $\vec e \models B_{Mary} (B_{Mary} wetFloor \land B_{Bob}\lnot wetFloor$). Thus, $wetFloor$ is a discrepancy between Bob and Mary's respective epistemic states from Mary's perspective.
\end{example}

\begin{definition}[Subjective Discrepancy-resolving Explanation]\label{def:disc_res_expl}
Given epistemic states $\vec e$ and a discrepancy $\beta$ between $e_i$ and $e_j$ from the perspective of agent $i$, we say that $\alpha$ is a discrepancy-resolving explanation for agent $j$ for $\beta$ from agent $i$'s perspective if $\vec e \models  B_i[\alpha]_j \lnot B_j \lnot\beta$.
\end{definition}


\begin{example}
A discrepancy-resolving explanation for $wetFloor$ for Bob from Mary's perspective is \textit{rain}.
\end{example}

Note that Definition~\ref{def:disc_res_expl} appeals to the weaker inconsistency-resolving explanation defined in Definition~\ref{def:might_expl}. Thus, the explainer need not find an $\alpha$ that it believes will allow the explainee to entail the discrepancy. Rather, $\alpha$ should resolve the discrepancy by explaining its possibility.





We cast agent $i$ as the explainer and agent $j$ as the explainee, and distinguish between two types of subjective discrepancies: (1) where $\beta$ is a discrepancy between $e_i$ and $e_j$ from the explainer's perspective; and (2) where $\beta$ is a discrepancy between $e_i$ and $e_j$ from the explainee's perspective. In (1), as discussed, the explainer (e.g., Mary) may provide a discrepancy-resolving explanation for $\beta$ (e.g., \textit{rain}). However, for (2), in order to provide such as explanation the explainer must \textit{believe} that the explainee believes that there exists a discrepancy between $e_i$ and $e_j$. If the explainer's beliefs about the explainee's beliefs are incomplete or incorrect, the former may not recognize that such a discrepancy exists. 


\subsubsection*{Explainer as Mediator}
Definition~\ref{def:subj_disc} can be straightforwardly generalized to capture a setting where agent $i$ believes that there exists a discrepancy between $e_j$ and $e_k$:
\begin{align*}
\vec e \models B_i(B_j\beta \land B_k \lnot\beta)
\end{align*}

Agent $i$ may also believe that agent $j$ believes that $\alpha$ is an explanation for $\beta$ for agent $k$, while also believing that $\alpha$ is not in fact a valid explanation for agent $k$ due to the discrepancy between the epistemic states of agents $j$ and $k$:
\begin{align*}
\vec e \models B_i(B_j\expl(k,\alpha,\beta) \land \lnot\expl(k,\alpha,\beta))
\end{align*}

Using Definition~\ref{def:multi_agent_expl}, agent $i$ may explain the discrepancy to agents $j$ and $k$.
Note that the notion of discrepancy discussed here can easily be extended to encode other, possibly richer notions of discrepancy including the degree to which the epistemic states of two agents are discrepant.

\section{The (In)Adequacy of the Explainer's Beliefs}~\label{sec:adeq}
The explainer is limited by the accuracy of its beliefs about the explainee's beliefs and reasoning capabilities. Specifically, the explainer's beliefs about the explainee's beliefs and reasoning capabilities must be accurate `enough' -- \textit{adequate} -- for the explainer to generate `good' explanations wrt the explainee.


\begin{definition}[Adequacy]\label{def:adequacy}
Given epistemic states $\vec e$ and explanandum $\beta$\eat{the epistemic states of agents $i$ and $j$, respectively,}, we say that agent $i$'s epistemic state $e_i$ is adequate wrt agent $j$ iff for all $\alpha$, $\vec e$ $\models$ $\bexpl(j,\alpha,\beta)$ iff $\vec e$ $\models$ $\expl(j,\alpha,\beta)$.
\end{definition}

That is, if agent $i$'s epistemic state is adequate wrt agent $j$ and $\beta$, then it can generate all explanations (for $\beta$) for agent $j$ that are also explanations for agent $j$ in its actual epistemic state, $e_j$.

\begin{theorem}
Given epistemic states $\vec e$, explanandum $\beta$ and $\preceq_{i,j},\preceq_j$, agent $i$'s perspective of agent $j$'s preference relation and agent $j$'s actual preference relation, respectively, if $\preceq_{i,j}$ $=$ $\preceq_j$ and $e_i$ is adequate wrt agent $j$ and $\beta$, then for all $\alpha$, $\alpha$ is an optimal explanation for agent $j$ from agent $i$'s perspective wrt $\preceq_{i,j}$ iff $\alpha$ is an optimal explanation for agent $j$ wrt $\preceq_{j}$.


\end{theorem}

That is, when $e_i$ is adequate wrt agent $j$ and when agent $i$'s beliefs about agent $j$'s preference relation are correct, the optimal explanation for agent $j$ from the perspective of agent $i$ is also the optimal obejctive explanation for agent $j$. The proof follows straightforwardly from Definitions \ref{def:optimal_expl} and \ref{def:adequacy}.


\subsection{Sources of (In)Adequacy}
Since most agents do not have a perfect image of another agent's mental state, an agent's beliefs about another agent may be inadequate for a myriad of reasons, including the inaccuracy of an agent's beliefs about the beliefs of other agents and about the way in which other agents revise their beliefs and perform entailment. In what follows, we focus on a setting where an agent holds inadequate beliefs about another agent's beliefs and illustrate using our running example.




\begin{example}
Returning to our example, assume that Mary forgot that Bob found the hole with her and so she now \textit{falsely} believes that Bob believes that there is no hole in the roof (i.e., $\vec e \models B_{Mary} B_{Bob}\lnot holeInRoof$). Mary will therefore believe that $\textit{rain} \land \textit{holeInRoof}$ is the minimal explanation for Bob (relative to an intuitive measure of minimality). Notice, however, that in her explanation, Mary is conveying more information than is needed for Bob to entail $\textit{wetFloor}$ (thereby violating Grice's quantity maxim).
\end{example}

\begin{example}
Now consider that Mary \textit{falsely} believes that Bob believes that it had rained and that there is no hole in the roof (perhaps she confused him with Tom!). Mary will therefore believe that $\textit{holeInRoof}$ is an explanation for Bob. However, $\vec e \not\models \expl(Bob,\textit{holeInRoof},wetFloor)$ since Bob does not believe \textit{rain}. This time, Mary has violated the quantity maxim by not providing \textit{enough} information for Bob to entail \textit{wetFloor}.
\end{example}

\begin{example}
Mary now falsely believes that Bob believes \textit{wetFloor} (i.e., $\vec e \models B_{Mary}B_{Bob}wetFloor$) and so does not provide him with an explanation, believing he does not require one. In this case, while \textit{wetFloor} is an objective discrepancy between Bob and Mary's epistemic states, it is not a discrepancy from Mary's perspective due to her false beliefs. 
\end{example}

\subsection*{Addressing Inadequacy}


It is possible to mitigate for the inadequacy of the explainer's beliefs in a variety of ways. For example, it may be beneficial for the explainer to attempt to refine its beliefs about the beliefs of the explainee when explanations are not understood by the explainee.\eat{ about the explainee's beliefs when perceiving a low degree of rationality in the explainee's behavior (when recognizing another agent's behavior) or }
To this end, the explainer could try to gather additional pertinent information by acting in the world (e.g., querying the explainee).
%
%
Additionally, \citeauthor{SreedharanHMK19} \cite{SreedharanHMK19} propose a learning technique which enables an explainer to learn a simple model of an explainee and decide, based on the learned model, what information would constitute a good explanation.
Further, \citeauthor{SreedharanCK18} \cite{SreedharanCK18} show how an explainer may generate explanations that are applicable to a set of possible explainee models which arise as the consequence of explainer uncertainty pertaining to the explainee's model.

Finally, while we emphasized the importance of the explainer modelling the beliefs of the explainee, our general account could in theory support the explainee, perhaps compensating for the explainer's inadequate beliefs, reasoning about the beliefs of the explainer to understand a given explanation that might otherwise be construed as inadequate. 
For example, consider \citeauthor{chandrasekaran2017takes}'s \cite{chandrasekaran2017takes} discussion of a Theory of AI's Mind where a human attempting to better understand a black-box decision making system can do so by familiarizing themselves with the system's capabilities, peculiarities, and shortcomings.

\section{Related Work}\label{sec:related_work}



%
As previously discussed, we are not the first to propose an account of explanation in terms of the epistemic state of an agent.
 \citeauthor{levesque1989knowledge} presents a knowledge-level account of abduction based on the epistemic state of an agent \cite{levesque1989knowledge}. He provides a generic definition of explanation that does not commit to a specific type of agent belief. Then, building on his seminal work on a logic of implicit and explicit belief \cite{levesque1984logic}, he shows how such different formal models of belief lead to different forms of abductive inference and resultant explanations. 
%
%
\citeauthor{boutilier1995abduction} \cite{boutilier1995abduction} similarly appeal to epistemic states to characterize the beliefs of an agent, employing belief revision to allow for explanations that are inconsistent with the epistemic state of the explainee.  Prior to the works of \citeauthor{levesque1989knowledge} and \citeauthor{boutilier1995abduction}, \citeauthor{gardenfors1988knowledge} \cite{gardenfors1988knowledge} proposed a model of explanation where explanations are defined relative to the epistemic states of agents. While \citeauthor{gardenfors1988knowledge}'s account is probabilistic, the models proposed by \citeauthor{levesque1989knowledge} and \citeauthor{boutilier1995abduction} are qualitative.
We share the use of epistemic states with all three works, the appeal to qualitative criteria with \citeauthor{levesque1989knowledge} and \citeauthor{boutilier1995abduction}, and the recognition of the importance of belief revision with \citeauthor{boutilier1995abduction}. Nevertheless, these works all characterize explanation with respect to a single agent providing no account of the distinct beliefs of the explainee \emph{and} explainer, nor do they capture their Theory of Mind. 

\citeauthor{nepomuceno2017abductive} \cite{nepomuceno2017abductive} propose an account of explanation that also recognizes the importance of a revision operator and the use of epistemic states. However, while their Dynamic Epistemic Logic (DEL) based framework can capture multiple agents, their focus remains on an agent's task of obtaining an abductive explanation for itself, rather than for other agents.


\citeauthor{halpern2005causes} \cite{halpern2005causes} proposed a structural model of explanation selection based on the epistemic state of the explainee. In their work, the explainee's epistemic state comprises a set of situations the explainee considers possible and an explanation is then meant to remove some of these possible situations such that the cause of some explanandum may be uniquely identified. 
%
Miller extends Halpern and Pearl's approach to include \textit{contrastive} explanations which are given relative to some counterfactual (e.g, in response to the question `\textit{Why P rather than Q?}') \cite{miller2018contrastive}. \citeauthor{halpern2005causes}, however, do not dicuss some of the necessary elements of Theory of Mind in explanation, such as the notions of explainer-explainee discrepancies and the adequacy of the explainer's beliefs.




In the context of XAIP, \citeauthor{SreedharanHMK19} \cite{SreedharanHMK19} demonstrate how the model reconciliation paradigm, proposed by \citeauthor{chakraborti2017plan} \cite{chakraborti2017plan}, can be generalized to address the important case where the explainee's model of the explainer's planning model is not explicitly known or not provided in a declarative form. Our work captures some of the insights in \citeauthor{SreedharanHMK19}'s work, in addition to incorporating the notions of epistemic states and belief revision, which in turn allows us to draw inspiration from the rich body of previous work in the field where these ideas originated.

We have focused discussion on the subset of work that is most closely related to the contributions of the paper. For a comprehensive survey of research on explanation, the reader is directed to 
\cite{DBLP:journals/ai/Miller19}.






\section{Concluding Remarks}\label{sec:concluding}

The use of Theory of Mind in explanation holds the promise of producing high-quality explanations that are tailored to the beliefs of the explainee, in the context of the beliefs (and ignorance) of the explainer.
In this paper, we identified a set of desiderata for explanation that utilizes Theory of Mind. These desiderata informed our proposed belief-based account of explanation. Key features of this account are the appeal to epistemic states to capture the mental states of \emph{both} the explainer and explainee, and the use of the explainee's belief revision to assimilate explanations.
Further, we formalized and discussed the notion of a discrepancy as a property that allows the explainer to anticipate and provide explanations without prompting. We also presented properties relating to the adequacy of the explainer's beliefs with respect to providing an explanation. 



This paper has provided a general characterization of explanation without focusing on its computational realization. This was done by design to allow for a diversity of explanation scenarios and agent types, including human, black-box decision maker, or knowledge-based system. Nevertheless in the simplest case if the beliefs of the explainer are represented as formulae (logical or probabilistic) then, as observed by \citeauthor{levesque1989knowledge} \cite{levesque1989knowledge} and \citeauthor{boutilier1995abduction} \cite{boutilier1995abduction}, our notion of explanation may be realized via an augmentation of existing abductive reasoning systems such as Theorist \citeauthor{poole1989explanation} \cite{poole1989explanation}, for example.



Further, in much of this paper we have been relating our Theory of Mind characterization of explanation in the context of English-like statements (e.g., Mary \emph{telling} Bob that it had rained last night). However, if we turn to the broad endeavour of XAI that helped motivate our account, we note that an explanation 
can take on many different forms other than human-interpretable language (e.g., a set of weights in a neural network, select pixels, a gesture, a heightening of intensity in a region of an image).
At its core, an explanation is something that is conveyed by the explainer to the explainee (e.g., by telling, demonstrating, visualizing, etc.) in order to justify the latter's belief in
some explanandum. For example, by constructing a heat-map from a medical image, an otherwise black-box decision-making algorithm can highlight for the explainee the pixels that have most strongly supported its classification decision \cite{montavon2018methods}, thereby allowing the explainee to assimiliate this explanation into their beliefs and better interpret the system's decision. As has been argued in this paper, the decision-making system, acting as an explainer, should possess the ability to take the epistemic state of the explainee into account, tease apart the salient features required for the explainee to justify its belief in the explanandum, and present those to the explainee as an explanation. %
Some of these insights pertaining to explanations for black-box solvers are similarly echoed by \citeauthor{SreedharanHMK19} in the context of their model reconciliation paradigm \cite{SreedharanHMK19} (Section 2).
Our general account is intended to provide building blocks towards this broader XAI objective.

There are several take-aways from this paper that are worth highlighting. Explanations need not be consistent with an agent's beliefs. As such, contrary to most logical treatments of explanation, characterizations of explanation should involve a belief revision component, and not just the expansion of existing beliefs to include an explanation. Further, by providing a belief-based account of explanation that characterizes mental states in terms of epistemic states, and by allowing for epistemic states and revision operators to be realized in a diversity of forms from standard logical accounts, to computer programs, neural networks or human brains, we can capture the mental states of a myriad of different types of agents. Finally, by characterizing explanations in terms of the explainer's beliefs about the explainee's beliefs and revision operator, we can capture the role of Theory of Mind in explanation for a myriad of different types of agents.

\section*{Acknoweldgements}
The authors gratefully acknowledge funding from the Natural Sciences and Engineering Research Council of Canada (NSERC), the Canadian Institute for Advanced Research (CIFAR), and Microsoft Research.

%







%
%
%

\bibliographystyle{splncs}

\bibliography{bib}  
%




\end{document}